\documentclass[final,12pt]{colt2023} 

\title[BoTW Analysis for Linear Bandits with FTRL algorithm]{Best-of-three-worlds Analysis for Linear Bandits with Follow-the-regularized-leader Algorithm}
\usepackage{times}

\usepackage{times}
\usepackage{multirow}
 \usepackage{booktabs}
\usepackage{amsmath,amscd,amsbsy,amssymb,latexsym,url,bm,natbib}
\allowdisplaybreaks
\usepackage{cleveref}
\usepackage{verbatim}
\usepackage{color}
\usepackage{dsfont}
\usepackage{caption}
\usepackage{algorithm}
\usepackage{algorithmic}
\usepackage{threeparttable}
\usepackage{graphicx}
\usepackage{thmtools}
\usepackage{thm-restate}
\usepackage{mathtools}

\newcommand{\sets}[1]{\left\{ #1 \right\}}

\newcommand{\trace}{\mathrm{trace}}
\newcommand{\abs}[1]{\left| #1 \right|}
\newcommand{\inner}[1]{\left< #1 \right>}

\newcommand{\bracket}[1]{\left(#1\right)}
\newcommand{\norm}[1]{\left\| #1 \right\|}

\newcommand{\EE}[1]{\mathbb{E} \left[#1\right]}

\DeclareMathOperator*{\argmin}{argmin}

\mathchardef\mhyphen="2D

\coltauthor{%
 \Name{Fang Kong} \Email{fangkong@sjtu.edu.cn}\\
 \Name{Canzhe Zhao} \Email{canzhezhao@sjtu.edu.cn}\\
 \Name{Shuai Li}\thanks{Corresponding author.} \Email{shuaili8@sjtu.edu.cn}\\
 \addr John Hopcroft Center for Computer Science, Shanghai Jiao Tong University, Shanghai, China%
}

\begin{document}

\maketitle


\begin{abstract}
    The linear bandit problem has been studied for many years in both stochastic and adversarial settings. Designing an algorithm that can optimize the environment without knowing the loss type attracts lots of interest. \citet{LeeLWZ021} propose an algorithm that actively detects the loss type and then switches between different algorithms specially designed for specific settings. However, such an approach requires meticulous designs to perform well in all environments. Follow-the-regularized-leader (FTRL) is another type of popular algorithm that can adapt to different environments. This algorithm is of simple design and the regret bounds are shown to be optimal in traditional multi-armed bandit problems compared with the detect-switch type. Designing an FTRL-type algorithm for linear bandits is an important question that has been open for a long time. In this paper, we prove that the FTRL algorithm with a negative entropy regularizer can achieve the best-of-three-world results for the linear bandit problem. Our regret bounds achieve the same or nearly the same order as the previous detect-switch type algorithm but with a much simpler algorithmic design. 
\end{abstract}



\section{Introduction}




The linear bandit problem is a fundamental sequential decision-making problem, in which the learner needs to choose an arm to pull and will receive a loss of the pulled arm as a response in each round $t$ \citep{Auer02,AwerbuchK04,DaniHK07,DaniHK08,Abbasi-YadkoriPS11,ChuLRS11,BubeckCK12}. Particularly, the expected value of the loss of each arm is the inner product between the feature of this arm and an \textit{unknown} loss vector. 
The learning objective is to minimize the \textit{regret}, defined as the difference between the learner's cumulative loss and the loss of some best arm  in hindsight. There are two common environments characterizing the schemes of the loss received by the learner.
In the stochastic environment, the loss vector is fixed across all rounds and the learner only observes a stochastic loss 
sampled from a distribution with the expectation being the inner product between the feature vector of the pulled arm and the loss vector. 
In this line of research, it has been shown that the instance-optimal regret bound is of order $O(\log T)$ (omitting all the dependence on other parameters) \citep{lattimore2017end}.
In the adversarial environment, the loss vector can arbitrarily change in each round. In this case, the minimax-optimal regret has been proved to be $O(\sqrt{T})$ \citep{BubeckCK12}.

While several algorithms are able to achieve the above (near) optimal regret guarantees in either stochastic environment \citep{lattimore2017end} or adversarial environment \citep{BubeckCK12}, these algorithms require prior knowledge about the regime of the environments. If such knowledge is not known a priori, the algorithms tailored to the adversarial environment will have too conservative regret guarantees in a stochastic environment. Also, without knowing such knowledge, the algorithms tailored to the stochastic environment may have linear regret guarantees in the adversarial environment and thus completely fail to learn. To cope with these issues, recent advances have emerged in simultaneously obtaining the (near) optimal regret guarantees in both the stochastic and adversarial environments, which is called the best-of-both-worlds (BoBW) result. 
For the special case of multi-armed bandit (MAB) problems,
algorithms with (near) optimal regret in both adversarial and stochastic environments
have been proposed \citep{BubeckS12,SeldinS14,AuerC16,SeldinL17,WeiL18,ZimmertS19,Ito21Parameter}.
In general, these algorithms fall into two lines.
The first line is based on conducting statistical tests to distinguish between the stochastic and adversarial environments and use the specially designed algorithm for the detected setting \citep{BubeckS12,AuerC16}.
The other approach is finding an appropriate regularizer for the follow-the-regularized-leader (FTRL) type algorithm which is originally designed for adversarial bandits to achieve improved regret guarantee in the stochastic environment, without compromising the adversarial regret guarantee \citep{SeldinS14,SeldinL17,WeiL18,ZimmertS19,Ito21Parameter}.
Compared with the first line,
the second approach has better adaptivity since it does not need to detect the nature of the environment and is also shown to achieve the optimal result in both environments as well as the intermediate corruption setting \citep{ZimmertS19}, which are called the best-of-three-worlds (BoTW) results.
This approach also attracts great interest in more general problems including combinatorial MAB \citep{ZimmertLW19,Ito21}, graph feedback \citep{ErezK21,Ito2022nearly} and Markov decision processes (MDPs) \citep{JinL20,JinHL21}. 
But for the canonical linear bandit problem, only the detect-switch-based method \citep{LeeLWZ021} is available. 
A natural question is still open:


\textit{Does there exist an FTRL-type algorithm for linear bandits, with the same adaptivity to both the adversarial and (corrupted) stochastic environments?}

In this paper, we give an affirmative answer to the above question. We investigate the FTRL algorithm with a negative entropy regularizer and show that such an algorithm can simultaneously achieve $O(\log T)$ regret in the stochastic setting and $O(\sqrt{T})$ regret in the adversarial setting.
Our work provides a more concise algorithm for the BoTW problem in linear bandits given that only a detect-switch-based method was available previously. 
Table \ref{table:comparison} provides the comparison between our work and previous works. 
As can be seen, our results achieve the same or nearly the same results compared with \citet{LeeLWZ021}, but with a much simpler algorithmic design. 
On the other hand, our results also unleash the potential of the FTRL-type algorithms to solve problems with linear structures. 
As a preliminary step, we believe the analysis for linear bandit also contributes some new ideas for the BoTW problem in other online learning models with linear structures such as linear MDPs.

\section{Related Works}

 \begin{table*}[htb]
\centering
\begin{threeparttable}
\caption{Comparisons of regret bounds with state-of-the-art results in different environments for linear bandits. $T$ is the horizon, $D$ is the finite arm set, $C$ is the corruption level,  $\Delta_{\min}$ is the minimum loss gap between the optimal and sub-optimal arms, $\ell$ is the loss vector in the stochastic setting, and $c(D,\ell)$ is the solution to the optimization problem in the lower bound of stochastic linear bandits with the natural upper bound $d/\Delta_{\min}$.  }\label{table:comparison}
\begin{tabular}{lll}
\toprule 
 & Stochastic with corruptions & Adversarial \\\hline
\rule{0pt}{15pt} \citet{BubeckCK12} &    & $\displaystyle  O\bracket{\sqrt{dT\log |D|}}$  \\\hline 
\rule{0pt}{15pt} \citet{lattimore2017end} & $\displaystyle  O\bracket{c(D,\ell)\log T}~(C=0)$    &  \\\hline 
\rule{0pt}{15pt} \citet{LeeLWZ021} & $\displaystyle O\bracket{c(D,\ell)\log^2 T+C}$ & $\displaystyle  O\bracket{\sqrt{dT\log^2 T}}$  \\\hline
\rule{0pt}{15pt} Ours & $\displaystyle O\bracket{{d\log^2 T}/{\Delta_{\min}}+\sqrt{{Cd\log^2 T}/{\Delta_{\min}}}}$ &  $\displaystyle O\bracket{\sqrt{dT\log^2 T }}$ \\
\bottomrule 
\end{tabular}
\begin{tablenotes}
\item[1] 
Particular care is required when comparing the results in the corrupted setting due to differences in the detailed problem definitions.
\end{tablenotes}
\end{threeparttable}
\end{table*}

\paragraph{Linear bandits.}

Linear bandit is a fundamental model in online sequential decision-making problems. Its stochastic setting is originally proposed by \citet{AbeL99}, and first solved by \citet{Auer02} using the optimism principle. Subsequently, the regret guarantees of this problem have been further improved by several algorithms \citep{DaniHK08,RusmevichientongT10,Abbasi-YadkoriPS11,ChuLRS11}. The adversarial setting dates back to \citet{AwerbuchK04}. The first $O(\sqrt{T})$ expected regret bound is obtained by the Geometric Hedge algorithm proposed by \citet{DaniHK07}. Latter, \citet{AbernethyHR08} establish the first computationally efficient algorithm achieving $\widetilde{O}(\sqrt{T})$ regret guarantee based on the FTRL framework. 
\citet{BubeckCK12} further obtain the regret guarantee that is minimax optimal up to a logarithmic factor. 
The recent work of \citet{LeeLWZ021} achieves the (near) optimal instance-dependent regret and minimax optimal regret in stochastic and adversarial environments respectively, via establishing a detect-switch type algorithm.
For linear bandits in a corrupted stochastic setting, \citet{Li2019stochastic} first obtain an instance-dependent regret guarantee, with an additional corruption term depending on the corruption level $C$  linearly and on $T$ logarithmically. Subsequently, \citet{BogunovicL0S21} attain the instance-independent bound in the same setting. For the more general setting of  linear contextual
bandits where the arm set may vary in each round and can even be infinite, \citet{Zhao2021linear,0002HS22,henearly} establish the sublinear instance-independent and/or instance-dependent regret guarantees. Remarkably, \citet{henearly} obtain the near-optimal instance-independent regret guarantees for both the corrupted and uncorrupted cases with the leverage of a weighted ridge regression scheme.
Due to the differences in the problem settings, the comparisons among these works require particular care.


\paragraph{Best-of-both-worlds algorithms.} 
The best-of-both-worlds (BoBW) algorithms aim to obtain the (near) optimal theoretical guarantees in both the stochastic and adversarial environments simultaneously,
and have been extensively investigated for a variety of problems in the online learning literature, including MAB problem \citep{BubeckS12,SeldinS14, AuerC16,SeldinL17,WeiL18,ZimmertS19,Ito21Parameter}, linear bandit problem \citep{LeeLWZ021}, combinatorial MAB problem \citep{ZimmertLW19,Ito21}, online learning with graph feedback \citep{ErezK21,0002ZL22,Ito2022nearly,Rouyer2022nearoptimal}, online learning to rank \citep{ChenZL22},
the problem of prediction with expert advice \citep{GaillardSE14,RooijEGK14,LuoS15,MourtadaG19},
finite-horizon tabular Markov decision processes (MDPs) \citep{JinL20,JinHL21},
online submodular minimization \citep{Ito22}, and partial monitoring problem \citep{Tsuchiya2022bobw}.

In addition to investigating the BoBW results for stochastic and adversarial environments, devising algorithms that can simultaneously learn in the stochastic environment with corruptions also gains lots of research attention recently \citep{ZimmertS19,ErezK21,Ito21aOptimal,LeeLWZ021,Ito21}. Particularly, algorithms with (near) optimal regret guarantees in adversarial, stochastic, and corrupted stochastic environments are typically called the best-of-three-worlds (BoTW) algorithms. In this work, we are interested in achieving the BoTW results in the linear bandit problem, but without using a detect-switch type algorithm similar to that of \citet{LeeLWZ021}. Instead, we aim to achieve the BoTW results for linear bandits by devising an FTRL-type algorithm, whose key idea is to leverage the self-bounding property of regret. 
Such types of algorithms have been proposed to obtain BoBW/BoTW results 
in MABs \citep{WeiL18,ZimmertS19,Ito21Parameter,Ito21aOptimal}, 
online learning with graph feedback \citep{ErezK21},
combinatorial MAB problem \citep{ZimmertLW19,Ito21},
online learning to rank \citep{ChenZL22}, 
finite-horizon tabular MDPs \citep{JinL20,JinHL21}, online submodular minimization \citep{Ito22}, and partial monitoring problem \citep{Tsuchiya2022bobw}.


\section{Preliminaries}\label{sec:setting}

In this section, we present the basic preliminaries regarding the linear bandit problem, 
in both the adversarial environments and stochastic environments (with possibly adversarial corruptions).





We denote by $D\subseteq \mathbb{R}^d$ the \textit{finite} arm set with $d$ being the ambient dimension of the arm feature vectors \citep{lattimore2017end,LeeLWZ021}. 
Without loss of generality, we further assume that $D$ spans $\mathbb{R}^d$ and $\|x\|_2\leq 1$ for all $x\in D$ as previous works \citep{LeeLWZ021,BubeckCK12}. In each round $t$, the learner needs to select an arm $x_t\in D$. Meanwhile, there will be an \textit{unknown} loss vector $\theta_t\in \mathbb{R}^d$ determined by the environment. 
Without loss of generality, we assume $\norm{\theta_t}_2 \le 1$ for any $t$.
Subsequently, the loss $\ell_t(x_t)$ of the chosen arm $x_t$ will be revealed to the learner. Particularly, $\ell_t(x_t)=\left\langle x_t,\theta_t\right\rangle+\epsilon_t\left(x_t\right)$ in linear bandits, with $\epsilon_t\left(x_t\right)$ being some independent zero-mean noise. The performance of one algorithm is measured by the \textit{regret} (a.k.a., \textit{expected pseudo-regret}), defined as
\begin{align*}
    R(T)=\mathbb{E}\left[\sum_{t=1}^T\ell_{t}(x_t)-\sum_{t=1}^T\ell_{t}(x^\ast)\right]\,,
\end{align*}
where the expectation is taken over the randomness of both the loss sequence $\{\ell_t\}_{t=1}^T$ and the internal randomization of the algorithm and $x^\ast\in\argmin_{x\in D}\mathbb{E}\left[\sum_{t=1}^T\ell_{t}(x)\right]$ is the best arm in expectation in hindsight.

For linear bandits in \textit{adversarial environment}, 
$\theta_t$ is chosen by an adversary arbitrarily in each round $t$. In this work, we consider the non-oblivious adversary, 
which means that the choice of $\theta_t$ may potentially depend on the history of the learner's actions $\{x_{\tau}\}_{\tau=1}^{t-1}$ up to round $t-1$, in addition to the adversary's own internal randomization.
Also, for ease of exposition, we assume $\epsilon_t\left(x_t\right)=0$ for all $t$ in the adversarial environment.

The \textit{stochastic} and \textit{corrupted stochastic} linear bandits can be contained in the following adversarial regime with a self-bounding constraint which is first formalized by \citet{ZimmertS21}. 



\begin{definition}[Adversarial regime with a self-bounding constraint]\label{def:ori:selfbound}
Let $C\ge0$ and $\Delta\in[0,2]^{|D|}$. 
An environment is defined as the \rm{adversarial regime with a} $(\Delta, C, T)$ \rm{self-bounding constraint}, \textit{if for any algorithm, the regret satisfies }
\begin{align}\label{eq:self_bounding_def}
    R(T)\geq \sum_{t=1}^T\sum_{x\in D}\Delta(x)\mathbb{P}(x_t=x)-C\,.
\end{align}
\end{definition}
The stochastic environment satisfies $\theta_t=\theta$ for some fixed but \textit{unknown} $\theta$.
Hence it has regret $R(T)=\mathbb{E}\left[\sum_{t=1}^T \Delta(x_t)\right]$ with $\Delta(x)=\langle x-x^\ast,\theta\rangle$ for all $x\in D$. It follows that a stochastic environment is also an adversarial regime with a $(\Delta,0,T)$ self-bounding constraint.
For this case, we assume the optimal arm $x^*$ to be unique as previous works studying stochastic linear bandits \citep{lattimore2017end,LeeLWZ021} and BoTW problems \citep{BubeckS12,ZimmertS19,LeeLWZ021,0002ZL22,Ito2022nearly}.
For the corrupted stochastic environment, it is also satisfied that $\theta_t=\theta$ for some \textit{unknown} $\theta$, but the learner now receives loss $\ell_t(x_t)=\left\langle x_t,\theta \right\rangle+\epsilon_t\left(x_t\right)+c_t$, where $c_t$ is the corrupted value chosen by the adversary in each round $t$ and may depend on the historical information $\{(\ell_\tau,x_{\tau})\}_{\tau=1}^{t-1}$. 
We note that the corrupted stochastic environment considered in this work is slightly more general than that considered by \citet{LeeLWZ021}, since we do not impose any structural assumptions over $c_t$, while they further assume $c_t$ is linear in $x_t$. As aforementioned, the corrupted stochastic environment is also an instance of the adversarial regime with a self-bounding constraint. To see this, let $C=\sum_{t=1}^T|c_t|$ be the corruption level, indicating the total amount of the corruptions. Then  a corrupted stochastic environment has regret satisfying Eq. (\ref{eq:self_bounding_def}), and hence it is an instance of the adversarial regime with a $(\Delta, C, T)$ self-bounding constraint.

\section{Algorithm}\label{sec:alg}
In this section, we present the proposed algorithm for linear bandits in both adversarial and (corrupted) stochastic environments, and detail its pseudo-code in Algorithm \ref{alg}.

\begin{algorithm}[thb!]
    \caption{FTRL for Linear Bandits}\label{alg}
    \begin{algorithmic}[1]
    \STATE Input: regularizer $\sets{\psi_t}_{t}$, exploration rate $\sets{\gamma_t}_{t}$, arm set $D$ and its G-optimal design $\pi$  \label{alg:input}\\
        \FOR{each round $t=1,2,\cdots$}
            \STATE $q_t \in \argmin_{p\in \Delta(D)} \sets{ \sum_{s=1}^{t-1}\inner{\hat{\ell}_s, p} + \psi_t(p) }$ \label{alg:line1}
            \STATE $p_t = \gamma_t \cdot \pi + (1-\gamma_t)\cdot q_t$ \label{alg:line2}
            \STATE Sample action $x_t \sim p_t$ and observe the loss $\ell_t(x_t)$ \label{alg:line3}
            \STATE Update $\Sigma_t = \sum_{x\in D} p_t(x)xx^\top$; $\hat{\ell}_t(x) = x^\top \Sigma_t^{-1}x_t \ell_t(x_t)$ for all $x\in D$ \label{alg:line4}
    \ENDFOR
    \end{algorithmic}
\end{algorithm}

In general, our algorithm follows the basic idea of FTRL that chooses the arm which minimizes the cumulative loss of past rounds together with the regularizer.
Let $ \Delta(D)$ be the probability simplex over $D$, $\hat{\ell}_s$ be the loss estimate in round $s$ to be defined later and $\psi_t(p)$ be the regularizer.
At each round $t$, our algorithm computes the regularized leader $q_t$ by solving the optimization problem (Line \ref{alg:line1}).
In particular, different from the
complex (hybrid) regularizers used in previous works studying FTRL-type BoTW algorithms \citep{ZimmertLW19,Ito21,ErezK21}, we simply choose the negative Shannon entropy as the regularizer:
\begin{align*}
    \psi_t(p) = -\beta_t H(p) = \beta_t\sum_{x\in D} p(x)\ln p(x)\,,
\end{align*}
where $\beta_t$ is the time-varying learning rate, whose concrete value is postponed to Theorem \ref{thm:thmregret}.

Then, to permit enough exploration and control the variance of the loss estimates, we mix the regularized leader $q_t$ with a distribution $\pi$ of the G-optimal 
design over the arm set $D$ on Line \ref{alg:line2},
where $\gamma_t$ is a time-varying hyper-parameter to be chosen later.
In particular, the  G-optimal design distribution $\pi$ is chosen such that
$\pi\in\argmin_{\rho\in\Delta(D)} g(\rho)$, where $g(\rho)=\max _{x \in D}\|x\|_{V(\rho)^{-1}}^2$ and $V(\rho)=\sum_{x \in D} \rho(x) x x^{\top}$ for any $\rho\in\Delta(D)$.
 Subsequently, the algorithm samples action $x_t$ from the mixed distribution $p_t$, and observes the corresponding loss $\ell_t(x_t)$ (Line \ref{alg:line3}).
 

Denote by $\Sigma_t = \sum_{x\in D} p_t(x)xx^\top$ the covariance matrix in terms of $p_t$. At the end of each round $t$, we construct a least squares loss estimate $\hat{\ell}_t(x)$ for each $x\in D$ (Line \ref{alg:line4}). 

\subsection{Theoretical Results}

In this section, we provide the theoretical guarantees of Algorithm \ref{alg} and the corresponding discussions. 
The following theorem presents the regret bounds for Algorithm \ref{alg} in adversarial and (corrupted) stochastic environments.

\begin{restatable}{thm}{thmregret}\label{thm:thmregret}
Choose $\psi_t(p) =  -\beta_t H(p)$ and $\gamma_t = \min\sets{g(\pi)/\beta_t,1/2 }$ where
\begin{align*}
    \beta_t:= 2g(\pi)+\sqrt{d\ln T/\ln |D|}+\sum_{\tau=1}^{t-1} \frac{\sqrt{d\ln T/\ln |D|}}{\sqrt{1+(\ln |D|)^{-1}\sum_{s=1}^\tau H(q_s) }} 
\end{align*} 
and $H(p):=-\sum_{x\in D} p(x)\ln p(x)$ is the negative Shannon entropy of distribution $p$. 
Our Algorithm \ref{alg} satisfies that:
\begin{itemize}
    \item In the adversarial environment, regret can be upper bounded by
    \begin{align*}
        R(T) \le O\bracket{\sqrt{dT\ln T \ln(|D|T)}} \,.
    \end{align*}
    \item In the stochastic environment with corruption level $C$, the regret can be bounded by
    \begin{align}
        R(T) \le O \bracket{\frac{ d\ln T \ln (|D|T) }{ \Delta_{\min}} + \sqrt{\frac{Cd\ln T \ln (|D|T)}{\Delta_{\min}}}  } \,,
    \end{align}
\end{itemize}
    where $\Delta_{\min}=\min_{x\in D,x\neq x^*}\Delta(x)$ is the minimum suboptimality gap over all sub-optimal arms.

\end{restatable}



\paragraph{Importance of the results. }
The FTRL-type algorithm has been demonstrated to achieve BoTW results for many problems including MAB \citep{ZimmertS19,Ito21Parameter}, combinatorial MAB \citep{ZimmertLW19,Ito21}, graph feedback \citep{ErezK21,Ito2022nearly}, and MDP \citep{JinL20,JinHL21}. However, the question of whether this algorithm can achieve BoTW results for the canonical linear bandit problem remains open. 
To the best of our knowledge, Theorem \ref{thm:thmregret} provides the first BoTW results for linear bandits with FTRL algorithm. 
Although our results have not reached the optimal order in separate settings \citep{lattimore2017end,BubeckCK12}, the convergence result itself is significant because it not only provides a more concise algorithm for this problem with only detect-switch-based methods available previously but also discovers the potential of FTRL-type algorithms.

\paragraph{Technical challenge. }
While the FTRL algorithm with a negative entropy regularizer is commonly used in the adversarial linear bandit literature, discovering the optimal exploration rate in the stochastic setting and achieving BoTW results is challenging.  
Similar to previous FTRL algorithms for the multi-armed bandit problem \citep{ZimmertS19}, the instance-dependent regret in the stochastic setting mainly relies on the self-bounding inequality that lower bounds the regret by the selecting probability of arms.
However, using the product of the sub-optimality gap and the selection probabilities of each sub-optimal arm as the self-bounding inequality and following the FTRL technique for MAB \citep{ZimmertS19} introduce a dependence on the number of arms, which is not ideal in the linear bandit case.
By leveraging the linear structure of the loss estimators and using standard deduction of FTRL, we discover that the regret can be upper bounded using the dimension $d$, the learning rate, and the cumulative entropy of the output decision vector. The former reflects the problem hardness of linear structured problems, and the latter can be further exploited to obtain the cumulative selecting times of all sub-optimal arms. This observation leads us to define the self-bounding inequality (Lemma \ref{lem:selfbound}) on the selecting times of all sub-optimal arms and use the minimum gap $\Delta_{\min}$ to measure the problem hardness.

\begin{lemma}{(Self-bounding inequality)}\label{lem:selfbound}
Algorithm \ref{alg} with Definition \ref{def:ori:selfbound} satisfies that
\begin{align*}
    R(T) \ge \frac{1}{2}\EE{\sum_{t=1}^T (1-q_t(x^*))}\Delta_{\min} -  C \,.
\end{align*}
\end{lemma}

The proof is deferred to Section \ref{sec:proof:lemmas}. 
Such an observation connects the cumulative entropy with the regret definition and guarantees convergence in the stochastic setting.

To simultaneously achieve an $O(\sqrt{T})$ result in the adversarial setting and $O(\log T)$ result in the stochastic setting, we need to balance the choice of the learning rate. 
Based on the standard analysis of the FTRL algorithm, the regret can be decomposed as 
\begin{align*}
    R(T) \le O\bracket{\EE{\sum_{t=1}^T \bracket{\frac{1}{\beta_t}  + (\beta_{t+1} -\beta_t) H(q_{t+1})}} }\,. 
\end{align*}
In the adversarial setting, we hope to get an $O(\sqrt{T})$ regret, the natural choice is to select $\beta_t = O(\sqrt{t})$ as both $\sum_t 1/\sqrt{t} = O(\sqrt{t})$ and $\sum_t \bracket{\sqrt{t+1}-\sqrt{t}} = O(\sqrt{t})$. While in the stochastic setting, we aim to get an $O(\log T)$ regret, so it would be natural to choose $\beta_t = O(t)$ as $H(q_t)$ tends to be $0$ when the policy always selects the unique optimal arm. Inspired by the motivation that $H(q_t)$ can be regarded as a constant term in the adversarial setting  and a decreasing sequence that eventually tends to $0$ in the stochastic setting, we set the learning rate $\beta_t$ as $O(\sum_{\tau=1}^{t-1} {1}/{\sqrt{1+\sum_{s=1}^{\tau}H(q_s) }})$.

The cooperation of the specially designed self-bounding inequality  and the choice of $\sets{\beta_t}_{t}$ contributes to the final results. And here we mainly follow the technique of \citet{Ito2022nearly} to achieve these results. The detailed proof is deferred to the next section. 



\paragraph{Comparisons with concurrent works. }

There is also a concurrent work studying the FTRL-type algorithm for the BoTW problem in linear bandits \citep{ito2023best}. 
\citet{ito2023best} study the data-dependent regret for this problem. In the worst case, their algorithm achieves  $O\bracket{d^{3/2}\sqrt{T\log T}}$ regret in the adversarial setting and $O(d^3\log T/\Delta_{\min}+\sqrt{Cd^3\log T/\Delta_{\min}})$ regret in the corrupted stochastic setting. 
Compared with these results, our regret bound is $O(d^2/\log T)$ better in the stochastic setting and $O(d/\sqrt{\log T})$ better in the adversarial setting. 
We also note that their results are data-dependent bounds, which may be better in environments with certain characteristics.






\section{Theoretical Analysis}\label{sec:analysis}
This section provides the proof of Theorem \ref{thm:thmregret}. 
We first present some useful lemmas and the main proof would come later. 


We start from the following lemma which shows that the regret can be decomposed in the form of the cumulative entropy of exploitation probability $q_t$. 
\begin{lemma}\label{lem:regret:unified}{(Regret Decomposition)}
    By choosing $\beta_t = 2g(\pi)+c+\sum_{\tau=1}^{t-1} \frac{c}{\sqrt{1+\ln |D|^{-1} \sum_{s=1}^\tau H(q_s) }}$ and $\gamma_t = \min\sets{g(\pi)/\beta_t,1/2}$, the regret can be bounded as
    \begin{align}
    R(T)=& O\bracket{ \bracket{\frac{d\ln T}{c\sqrt{\ln |D|}}  + c\sqrt{\ln |D|}}\sqrt{\EE{\sum_{t=1}^T H(q_t)}} +(2g(\pi)+c)\ln |D| } \label{eq:main:regret:unified}\,.
    \end{align}
\end{lemma}

The next lemma further links the cumulative entropy with the expected selecting times of all sub-optimal arms. 

\begin{lemma}{(Relationship between $H(q_t)$ and expected selecting times of all sub-optimal arms)}\label{lem:relationBetweenHandRegret}
    \begin{align*}
        \sum_{t=1}^T H(q_t) \le \bracket{\sum_{t=1}^T(1-q_t(x^*))} \ln \frac{e|D|T}{\sum_{t=1}^T(1-q_t(x^*))}\,.
    \end{align*}
\end{lemma}

The detailed proof for both lemmas are deferred to Section \ref{sec:proof:lemmas}. 
Based on Lemma \ref{lem:regret:unified} and \ref{lem:relationBetweenHandRegret}, we are now able to upper bound the regret as below. 

    \begin{align}
        R(T) \le& O\bracket{ \bracket{\frac{d\ln T}{c\sqrt{\ln |D|}} +c\sqrt{\ln |D|}}   \sqrt{ \EE{\sum_{t=1}^T H(q_t)}} +(2g(\pi)+c)\ln |D|  } \notag \\
        \le & O\bracket{ \bracket{\frac{d\ln T}{c\sqrt{\ln |D|}} +c\sqrt{\ln |D|}}   \sqrt{ \EE{\underbrace{\bracket{\sum_{t=1}^T(1-q_t(x^*))} \ln \frac{e|D|T}{\sum_{t=1}^T(1-q_t(x^*))}}_{\text{term }A}} } +(2g(\pi)+c)\ln |D| } \notag \\
       \le&  O\bracket{ \bracket{\frac{d\ln T}{c\sqrt{\ln |D|}} +c\sqrt{\ln |D|}}   \sqrt{ \EE{\ln \bracket{e|D|T} \max\sets{e, \sum_{t=1}^T(1-q_t(x^*)) }  } } + (2g(\pi)+c)\ln |D| }  \label{eq:main:termA} \\
       =& O\bracket{ \bracket{\frac{d\ln T}{c\sqrt{\ln |D|}} +c\sqrt{\ln |D|}}   \sqrt{ \ln \bracket{e|D|T} \max\sets{e, \EE{\sum_{t=1}^T(1-q_t(x^*))} }   } +(2g(\pi)+c)\ln |D| } \notag\\
       =& O\bracket{ \sqrt{d\ln T} \cdot  \sqrt{ \ln \bracket{e|D|T} \max\sets{e, \EE{\sum_{t=1}^T(1-q_t(x^*))} }   } } \label{eq:main:choiceOfc} \,,
    \end{align}
    where Eq. \eqref{eq:main:termA} holds since term $A \le \sum_{t=1}^T(1-q_t(x^*)) \ln |D|T$ if $\sum_{t=1}^T(1-q_t(x^*)) > e$ and term $A \le e \ln e|D|^2T/(|D|-1)=O(e\ln \bracket{e|D|T})$ otherwise as $\sum_{t=1}^T(1-q_t(x^*)) \ge 1-q_1(x^*) = 1-1/|D|$; Eq. \eqref{eq:main:choiceOfc} is due to $c=\sqrt{d\ln T/\ln |D|}$.  

\paragraph{Regret analysis in the adversarial setting.}
    According to Eq. \eqref{eq:main:choiceOfc}, since $\sum_{t=1}^T (1-q_t(x^*)) \le T$, it holds that
    \begin{align*}
        R(T) \le& O\bracket{ \bracket{\sqrt{d\ln T}  \sqrt{ \ln \bracket{e|D|T} \cdot T   }  }} = O\bracket{\sqrt{dT\ln T \ln(|D|T)}} \,.
    \end{align*}

\paragraph{Regret analysis in the stochastic setting with corruptions.}
According to Lemma \ref{lem:selfbound}, 
    \begin{align}
        R(T) =& (1+\lambda)R(T) - \lambda R(T) \notag \\
        \le & O\left( (1+\lambda)\cdot \sqrt{d\ln T} \cdot  \sqrt{ \ln \bracket{e|D|T} \max\sets{e, \EE{\sum_{t=1}^T(1-q_t(x^*))} }   }  \right. \notag \\
        &\left. -\frac{1}{2}\lambda\EE{\sum_{t=1}^T (1-q_t(x^*))}\Delta_{\min} +\lambda C  \right) \notag \\
        \le&    O\bracket{    \frac{(1+\lambda)^2 d\ln T \ln (|D|T) }{2\lambda \Delta_{\min}} + \lambda C }\label{eq:main:abx} \\
        =& O\bracket{ \frac{d\ln T \ln (|D|T)}{2\lambda \Delta_{\min}} +  \frac{ d\ln T \ln (|D|T) }{ \Delta_{\min}} + \frac{\lambda  d^2\ln T \ln (|D|T) }{2\Delta_{\min}} + \lambda C } \notag \\
        \le& O \bracket{\frac{ d\ln T \ln (|D|T) }{ \Delta_{\min}} + \sqrt{\frac{Cd\ln T \ln (|D|T)}{\Delta_{\min}}}  } \label{eq:main:choiceOfLambda}\,,
    \end{align}
where Eq. \eqref{eq:main:abx} holds since $a\sqrt{x} - bx/2 = a^2/2b - 1/2\bracket{ a/\sqrt{b} - \sqrt{bx} } \le a^2/2b$
for any $a,b,x\ge 0$; Eq. \eqref{eq:main:choiceOfLambda} holds by choosing $\lambda = \sqrt{ \frac{ d\ln T \ln (|D|T)/(2\Delta_{\min})}{ d\ln T \ln (|D|T)/(2\Delta_{\min}) +C   } }$. 


\subsection{Proof of Useful Lemmas}\label{sec:proof:lemmas}

This section provides detailed proof for the lemmas used in the analysis. The first is the proof of Lemma \ref{lem:selfbound} which is about self-bounding inequality.

\begin{proof}[Proof of Lemma \ref{lem:selfbound}]
\begin{align*}
    R(T) &\ge \EE{\sum_{t=1}^T (\ell_t(x_t) - \ell_t(x^*)) } - C \\
    &= \EE{\sum_{t=1}^T \inner{{\ell}_t, p_t - p^*} } - C\\
    &\ge \EE{\sum_{t=1}^T (1-\gamma_t) \inner{{\ell}_t, q_t - p^*}   }  - C\\
    &= \EE{\sum_{t=1}^T (1-\gamma_t) \sum_{x} q_t(x)  \Delta(x)   } - C \\
    &\ge \frac{1}{2}\EE{\sum_{t=1}^T (1-q_t(x^*))} \Delta_{\min} - C\,,
\end{align*}
where the last inequality holds by choosing $\gamma_t \le 1/2$ for all $t$. 
\end{proof}

In the following, we provide the proof of Lemma \ref{lem:regret:unified} which is used in the analysis of Theorem \ref{thm:thmregret}. 

\begin{proof}[Proof of Lemma \ref{lem:regret:unified}]
\begin{align}
    R(T) =& \EE{\sum_{t=1}^T \ell_t(x_t) - \ell_t(x^*)} \notag \\
    =& \EE{\sum_{t=1}^T \inner{{\ell}_t, p_t - p^*} } \label{eq:main:dueToSample}\\
    =& \EE{\sum_{t=1}^T (1-\gamma_t) \inner{\hat{\ell}_t, q_t - p^*}   } + \EE{\sum_{t=1}^T  \gamma_t\inner{\ell_t, \pi - p^*}  }  \label{eq:main:unBiased}\\
    \le& \EE{\sum_{t=1}^T (1-\gamma_t) \inner{\hat{\ell}_t, q_t - p^*}   } + 2\sum_{t=1}^T  \gamma_t 
    \label{eq:main:bounded}\\
    \le& \EE{\sum_{t=1}^T (1-\gamma_t) \bracket{
        \inner{\hat{\ell}_t, q_t - q_{t+1}} - D_t(q_{t+1},q_t) + \psi_t(q_{t+1}) -   \psi_{t+1}(q_{t+1})} } \notag  \\
        &+ \psi_{T+1}(p^*) - \psi_1(q_1) + 2\sum_{t=1}^T  \gamma_t \label{eq:main:exercise} \\
     =& \EE{\sum_{t=1}^T (1-\gamma_t) \bracket{
        \inner{\hat{\ell}_t, q_t - q_{t+1}} - D_t(q_{t+1},q_t) + (\beta_{t+1}-\beta_t) H(q_{t+1)}} } \notag \\
        &+ \psi_{T+1}(p^*) - \psi_1(q_1) + 2\sum_{t=1}^T  \gamma_t \notag  \\
      \le& \EE{\sum_{t=1}^T (1-\gamma_t) \bracket{
        \underbrace{\inner{\hat{\ell}_t, q_t - q_{t+1}} - D_t(q_{t+1},q_t)}_{\text{part } 1} + (\beta_{t+1}-\beta_t) H(q_{t+1})} } +\beta_1 \ln |D|+ 2\sum_{t=1}^T  \gamma_t  \notag \\
        \le&  \EE{\sum_{t=1}^T  \bracket{
        \frac{3d}{\beta_t} + \underbrace{(\beta_{t+1}-\beta_t) H(q_{t+1})}_{\text{part }2} }} +\beta_1 \ln |D| \label{eq:main:boundPart1}\\
        \le& \EE{\sum_{t=1}^T  
        \frac{3d}{\beta_t} + 2c \sqrt{\ln |D| \sum_{t=1}^T H(q_t)} }  +(2g(\pi)+c)\ln |D|    \label{eq:main:boundPart2} \\
        \le& \EE{3d \sum_{t=1}^T \frac{\sqrt{1+(\ln |D|)^{-1}\sum_{s=1}^t H(q_s)} }{c\cdot t} + 2c \sqrt{\ln |D| \sum_{t=1}^T H(q_t)} }  +(2g(\pi)+c)\ln |D| \label{eq:main:beta} \\
        \le& O\bracket{ \bracket{\frac{d\ln T}{c\sqrt{\ln |D|}}+c\sqrt{\ln |D|} }
        \sqrt{\EE{\sum_{t=1}^T H(q_t)}} +(2g(\pi)+c)\ln |D| }  \,. \notag 
\end{align}
where Eq. \eqref{eq:main:dueToSample} holds due to $x_t \sim p_t$, Eq. \eqref{eq:main:unBiased} holds since $\hat{\ell}_t$ is an unbiased estimator, Eq. \eqref{eq:main:bounded} is because $\inner{\ell_t, \pi - p^*} \le 2$, Eq. \eqref{eq:main:exercise} is based on Lemma \ref{lem:decomposeFTRL}, Eq. \eqref{eq:main:boundPart1} holds according to Lemma \ref{lem:part1} and the choice of $\gamma_t =\min\sets{ g(\pi)/\beta_t, 1/2}\le d/\beta_t$ \citep{kiefer1960equivalence}, Eq. \eqref{eq:main:boundPart2} and \eqref{eq:main:beta} are due to Lemma \ref{lem:part2} and the choice of $\beta_t$, the last inequality is based on Jensen's inequality.  
\end{proof}

We next provide Lemma \ref{lem:decomposeFTRL},  \ref{lem:part1} and \ref{lem:part2} that is used in the proof of Lemma \ref{lem:regret:unified}.


\begin{lemma}{(Decomposition of FTRL, Exercise 28.12 in \citet{lattimore2020bandit})}\label{lem:decomposeFTRL}
    \begin{align*}
        \sum_{t=1}^T  \inner{\hat{\ell}_t, q_t - p^*} \le &\sum_{t=1}^T \bracket{
        \inner{\hat{\ell}_t, q_t - q_{t+1}} - D_t(q_{t+1},q_t) + \psi_t(q_{t+1}) -   \psi_{t+1}(q_{t+1})} \\
        &+ \psi_{T+1}(p^*) - \psi_1(q_1)\,.
    \end{align*}
\end{lemma}

\begin{lemma}{(Bound Part 1)}\label{lem:part1}
    \begin{align*}
        \EE{ (1-\gamma_t) \bracket{
        \inner{\hat{\ell}_t, q_t - q_{t+1}} - D_t(q_{t+1},q_t)} } \le \frac{d}{\beta_t} \,.
    \end{align*}
\end{lemma}

\begin{proof}
    \begin{align}
       & \EE{ (1-\gamma_t) \bracket{
        \inner{\hat{\ell}_t, q_t - q_{t+1}} - D_t(q_{t+1},q_t)} }  \notag\\ 
        \le &  \EE{  (1-\gamma_t)\beta_t \bracket{ \sum_{x\in D}q_t(x)\bracket{ \exp(-\hat{\ell}_t(x)/\beta_t ) +\hat{\ell}_t(x)/\beta_t -1  }  } } \label{eq:lemma8inIto}\\
        \le& \EE{  (1-\gamma_t)\beta_t \bracket{ \sum_{x\in D}q_t(x)\bracket{ \hat{\ell}^2_t(x)/\beta^2_t  }  } } \label{eq:exp}\\
        =& \EE{  \frac{(1-\gamma_t)}{\beta_t} \bracket{ \sum_{x\in D}q_t(x)\hat{\ell}^2_t(x)   } }  \notag \\
        \le& \EE{ \frac{1}{\beta_t} \bracket{ \sum_{x\in D}p_t(x)\hat{\ell}^2_t(x)   } }  \label{eq:duetoqt} \,,
    \end{align}
where Eq. \eqref{eq:lemma8inIto} is based on \citet[Lemma 8]{Ito2022nearly}, Eq. \eqref{eq:exp} holds due to $\exp(-a)+a-1\le a^2$ when $a\ge -1$ and Lemma \ref{lem:ge-1}, Eq. \eqref{eq:duetoqt} is from the definition of $p_t$. 


    Recall that $\hat{\ell}_t(x) = x^\top \Sigma_t^{-1} x_t \ell_t(x_t)$. We have
    \begin{align*}
        \hat{\ell}_t(x)^2 = \ell_t(x_t)^2 x_t^\top  \Sigma_t^{-1} x x^\top \Sigma_t^{-1} x_t \le x_t^\top  \Sigma_t^{-1} x x^\top \Sigma_t^{-1} x_t 
    \end{align*}
    and further, 
    \begin{align*}
        \sum_x p_t(x) \hat{\ell}_t(x)^2 \le \sum_x p_t(x)x_t^\top  \Sigma_t^{-1} x x^\top \Sigma_t^{-1} x_t  = x_t^\top  \Sigma_t^{-1} x_t = \trace(x_tx_t^\top \Sigma_t^{-1}) \,.
    \end{align*}
    Above all, 
    \begin{align*}
       \EE{ \sum_x p_t(x) \hat{\ell}_t(x)^2} \le \trace\bracket{\sum_x p_t(x) x x^\top \Sigma_t^{-1}} = d \,.
    \end{align*}
    Lemma \ref{lem:part1} is then proved. 
\end{proof}

\begin{lemma}\label{lem:ge-1}
    Denote $\Sigma(\pi) = \sum_{x\in D}\pi(x) xx^\top$ and $g(\pi) = \max_{x\in D} x^\top \Sigma(\pi)^{-1} x$. By choosing $\gamma_t =
    \min\sets{g(\pi)/\beta_t,1/2}
    $, we have $\hat{\ell}_t(x)/\beta_t \ge -1$ for all $x\in D$. 
\end{lemma}

\begin{proof}
Since $\Sigma_t = \sum_x p_t(x) xx^\top = \sum_x \bracket{ (1-\gamma_t)q_t(x) + \gamma_t \pi(x)} xx^\top \succeq \sum_x \gamma_t \pi(x) xx^\top = \gamma_t\Sigma(\pi) $, we have 
\begin{align*}
    \Sigma_t^{-1} \preceq \frac{1}{\gamma_t} \Sigma(\pi)^{-1} \,.
\end{align*}
Further, 
\begin{align*}
    \abs{x^\top \Sigma_t^{-1}x_t } \le \norm{x}_{\Sigma_t^{-1}} \norm{x_t}_{\Sigma_t^{-1}} \le \max_{v\in D} v^\top \Sigma_t^{-1}v \le \frac{1}{\gamma_t} \max_{v\in D} v^\top \Sigma(\pi)^{-1}v = \frac{g(\pi)}{\gamma_t} \,.
\end{align*}
With the choice of $\gamma_t$, it holds that
\begin{align*}
    \abs{ \frac{\hat{\ell}_t(x)}{\beta_t}  } = \abs{ \frac{1}{\beta_t} x^\top \Sigma_t^{-1} x_t \ell_t(x_t)  } \le \abs{ \frac{1}{\beta_t} x^\top \Sigma_t^{-1} x_t  } \le \frac{g(\pi)}{\beta_t \gamma_t} \le 1\,,
\end{align*}
The last inequality is due to the choice of $\beta_t$. 

\end{proof}

\begin{lemma}{(Bound Part 2)}\label{lem:part2}
    By choosing $\beta_t = 2g(\pi)+c+\sum_{\tau=1}^{t-1} \frac{c}{\sqrt{1+\ln |D|^{-1} \sum_{s=1}^\tau H(q_s) }}$,
    \begin{align*}
        \sum_{t=1}^T (\beta_{t+1}-\beta_t)H(q_{t+1}) \le 2c \sqrt{\ln |D| \sum_{t=1}^T H(q_t)} \,.
    \end{align*}
\end{lemma}

\begin{proof}
    Based on the definition of $\beta_t$, it holds that
    \begin{align*}
        \sum_{t=1}^T (\beta_{t+1}-\beta_t)H(q_{t+1}) = &\sum_{t=1}^T \frac{c}{\sqrt{1+\ln |D|^{-1} \sum_{s=1}^t H(q_s) }} H(q_{t+1}) \\
        \le& \sqrt{\ln |D|}\sum_{t=1}^T \frac{c}{\sqrt{\ln |D|+ \sum_{s=1}^t H(q_s) }} H(q_{t+1}) \\
        =& 2\sqrt{\ln |D|}\sum_{t=1}^T \frac{c}{\sqrt{\ln |D|+ \sum_{s=1}^t H(q_s) } + \sqrt{\ln |D|+ \sum_{s=1}^t H(q_s) }} H(q_{t+1})\\
        \le& 2\sqrt{\ln |D|}\sum_{t=1}^T \frac{c}{\sqrt{ \sum_{s=1}^{t+1} H(q_s) } + \sqrt{ \sum_{s=1}^t H(q_s) }} H(q_{t+1}) \\
        =& 2c\sqrt{\ln |D|}\sum_{t=1}^T \bracket{ \sqrt{ \sum_{s=1}^{t+1} H(q_s) } - \sqrt{ \sum_{s=1}^t H(q_s) }} \\
        =& 2c\sqrt{\ln |D|} \bracket{ \sqrt{ \sum_{s=1}^{T+1} H(q_s) } - \sqrt{  H(q_1) }} \\
        \le&  2c\sqrt{\ln |D|}  \sqrt{ \sum_{t=1}^{T} H(q_t) } \,,
    \end{align*}
    where the last two inequalities holds since $H(q_t) \le H(q_1) = \ln |D|$ for any $t$. 
\end{proof}

And the last is the proof of Lemma \ref{lem:relationBetweenHandRegret} that is used in the analysis of Theorem \ref{thm:thmregret}. 

\begin{proof}[Proof of Lemma \ref{lem:relationBetweenHandRegret}]
    According to the definition of $H(p)$ for a distribution $p$ over $D$ and the fact that the optimal arm $x^* \in D$ is unique, we can bound $H(p)$ as
    \begin{align*}
        H(p) = \sum_x p(x)\ln \frac{1}{p(x)} =&  p(x^*)\ln \frac{p(x^*)+1-p(x^*)}{p(x^*)} + \sum_{x\neq x^*}p(x)\ln \frac{1}{p(x)} \\
        \le & p(x^*)\frac{1-p(x^*)}{p(x^*)} + \sum_{x\neq x^*}p(x)\ln \frac{1}{p(x)} \\
        \le & 1-p(x^*) + (|D|-1) \frac{\sum_{x\neq x^*}p(x)}{|D|-1}\ln \frac{|D|-1}{\sum_{x\neq x^*}p(x)} \\
        =& (1-p(x^*))\bracket{ 1+\ln \frac{|D|-1}{(1-p(x^*))} } \,,
    \end{align*}
    where the last inequality is due to Jensen's inequality. Thus, 
    \begin{align*}
        \sum_{t=1}^T H(q_t) \le& \sum_{t=1}^T (1-q_t(x^*))\bracket{ 1+\ln \frac{|D|-1}{(1-q_t(x^*))} } \\
        \le& \bracket{\sum_{t=1}^T (1-q_t(x^*))}\bracket{ 1+\ln \frac{T(|D|-1)}{\sum_{t=1}^T(1-q_t(x^*))} } \\
        \le& \bracket{\sum_{t=1}^T(1-q_t(x^*))} \ln \frac{e|D|T}{\sum_{t=1}^T(1-q_t(x^*))} \,.
    \end{align*}
\end{proof}


\section{Conclusion}\label{sec:conclusion}

In this work, we aim to provide the theoretical analysis for an FTRL-type algorithm when solving the BoTW problem in linear bandits. 
We investigate the performances of the FTRL algorithm with the common negative entropy regularizer. 
By discovering its difference in different environments and defining the self-bounding inequality to cooperate with it, we show that the algorithm can achieve $O(\log T)$ regret in the stochastic setting and $O(\sqrt{T})$ regret in the adversarial setting. 
To the best of our knowledge, it is the first BoTW analysis for linear bandits with an FTRL-type algorithm. Our work not only provides a simpler algorithm for this problem compared with the previous detect-switch-based method \citep{LeeLWZ021}, but also taps the potential of FTRL when dealing with sequential decision-making problems with linear structures. 
As a special case of linear MDP, we believe our analysis for linear bandits also provides new insights into the BoTW problem for this general model.

Stochastic linear bandit problem has a special property that the optimal result depends on $c(D,\ell)$ which is the solution to an optimization problem and has no explicit form \citep{lattimore2017end}.
A regret bound that adapts to $c(D,\ell)$ would better illustrate the problem's structure.
\citet{LeeLWZ021} achieve this since they adopt a detect-switch-based method and the specially designed algorithm for the stochastic setting helps. 
One interesting future direction would be to investigate such optimal exploration rates for FTRL-type algorithms and derive regret bounds in the stochastic setting that enjoys the form of $c(D,\ell)$.








\acks{We thank Shinji Ito for valuable discussions and suggestions on the improvement of the learning rate. 
}

\bibliography{ref}

\appendix

\end{document}